\documentclass{article}

 \usepackage[nonatbib,preprint]{neurips_2025}

\usepackage[utf8]{inputenc} % allow utf-8 input
\usepackage[T1]{fontenc}    % use 8-bit T1 fonts
\usepackage{hyperref}       % hyperlinks
\usepackage{url}            % simple URL typesetting
\usepackage{booktabs}       % professional-quality tables
\usepackage{amsfonts}       % blackboard math symbols
\usepackage{nicefrac}       % compact symbols for 1/2, etc.
\usepackage{microtype}      % microtypography
\usepackage{xcolor}         % colors

\usepackage[style=numeric,firstinits,natbib,maxbibnames=99]{biblatex}
\definecolor{mydarkblue}{rgb}{0.0, 0.0, 0.7}
\hypersetup{ 
	pdftitle={},
	pdfkeywords={},
	pdfborder=0 0 0,
	pdfpagemode=UseNone,
	colorlinks=true,
	linkcolor=mydarkblue,
	citecolor=mydarkblue,
	filecolor=mydarkblue,
	urlcolor=mydarkblue,
	pdfview=FitH,
	pdfauthor={Anonymous},
}

\usepackage{hyperref}
\usepackage{url}
\usepackage{comment}
\usepackage{subcaption}

\usepackage{amsmath,amsfonts,bm,amsthm}
\usepackage{bbm}
\usepackage{mathtools}

\newcommand{\supp}{\textnormal{supp}}
\renewcommand{\subset}{\subseteq}

\def\1{\bm{1}}
\def\va{{\bm{a}}}

\def\vf{{\bm{f}}}
\def\vg{{\bm{g}}}

\def\vv{{\bm{v}}}

\def\vx{{\bm{x}}}

\DeclareMathAlphabet{\mathsfit}{\encodingdefault}{\sfdefault}{m}{sl}
\SetMathAlphabet{\mathsfit}{bold}{\encodingdefault}{\sfdefault}{bx}{n}

\def\gA{{\mathcal{A}}}

\def\gF{{\mathcal{F}}}
\def\gG{{\mathcal{G}}}

\def\gQ{{\mathcal{Q}}}

\def\gX{{\mathcal{X}}}

\def\sE{{\mathbb{E}}}
\def\sR{{\mathbb{R}}}

\newcommand{\normin}[1]{\lVert #1 \rVert}

\renewcommand{\subset}{\subseteq}

\DeclarePairedDelimiterX{\infdivx}[2]{(}{)}{%
  #1\;\delimsize\|\;#2%
}

\usepackage{cleveref}
\usepackage{thm-restate}
\usepackage{wrapfig}
\usepackage{graphicx}

\crefname{example}{Example}{Examples}
\Crefname{example}{Example}{Examples}
\crefname{theorem}{Theorem}{Theorems}
\Crefname{theorem}{Theorem}{Theorems}
\crefname{lemma}{Lemma}{Lemmas}
\Crefname{lemma}{Lemma}{Lemmas}
\crefname{remark}{Remark}{Remarks}
\Crefname{remark}{Remark}{Remarks}
\crefname{proposition}{Proposition}{Propositions} 
\Crefname{proposition}{Proposition}{Propositions} 
\crefname{corollary}{Corollary}{Corollaries}
\Crefname{corollary}{Corollary}{Corollaries}
\crefname{definition}{Definition}{Definitions}
\Crefname{definition}{Definition}{Definitions}
\crefname{assumption}{Assumption}{Assumptions}
\Crefname{assumption}{Assumption}{Assumptions}
\Crefname{equation}{}{}
\crefname{equation}{}{}

\newtheorem{example}{Example}
\newtheorem{theorem}{Theorem}
\newtheorem{lemma}[theorem]{Lemma}
\newtheorem{remark}{Remark}
\newtheorem{proposition}[theorem]{Proposition} 

\newtheorem{definition}{Definition}
\newtheorem{desideratum}{Desideratum}

\newtheorem{assumption}{Assumption}

\title{On the Identifiability of Latent Action Policies}
\workshoptitle{UniReps: 3rd Edition of the Workshop on Unifying Representations in Neural Models}

\author{
Sébastien Lachapelle\\
  Samsung AI Lab, Montreal\\
  \texttt{s.lachapelle@samsung.com} \\
}

\begin{document}

\maketitle

\begin{abstract}
  We study the identifiability of \textit{latent action policy learning} (LAPO), a framework introduced recently to discover representations of actions from video data. We formally describe desiderata for such representations, their statistical benefits and potential sources of unidentifiability. Finally, we prove that an entropy-regularized LAPO objective identifies action representations satisfying our desiderata, under suitable conditions. Our analysis provides an explanation for why \textit{discrete} action representations perform well in practice.%Our analysis partly explains why \textit{discrete} action representations are crucial in practice.
\end{abstract}

\section{Introduction \& background}\label{sec:intro}
In robot control, \textit{behavior cloning} is an approach to learn a policy from action-labeled expert trajectories~\citep{NIPS1988_812b4ba2,foster2024is}. It simply consists in training a policy $\pi(a \mid \vx)$ via supervised learning on data of the form $\{(\vx_i, a_i)\}_{i=1}^N$ where $a_i$ is the action taken by the expert policy in state $\vx_i$. The success of the approach relies on at least two aspects: (i) the demonstrations $(\vx_i, a_i)$ have to be generated from a sufficiently good expert policy, and (ii) the number of demonstrations $N$ has to be sufficiently large to make learning possible. Although conceptually simple, the framework requires a large amount of action-labeled expert demonstrations to be successful, which can be costly to acquire.

To address this issue, \citet{schmidt2024learning} introduced \textit{latent action policy learning} (LAPO) which can leverage large corpuses of unannotated video data in order to reduce reliance on action-labeled expert trajectories. LAPO proceeds in three stages. \textbf{First}, given a large dataset of state/next-state pairs $\{(\vx_i, \vx_i')\}_{i=1}^N$, LAPO minimizes the reconstruction loss 
\begin{align*}
    \textstyle \frac{1}{N}\sum_{i=1}^N \sE_{\hat a \sim \hat q(\hat a \mid \vx_i, \vx_i')}\normin{\vx'_i - \hat\vg(\vx_i, \hat a)}^2_2\,,
\end{align*}
where the \textit{inverse dynamics model} (IDM) $\hat q(\hat a \mid \vx, \vx')$ encodes pairs $(\vx, \vx')$ into action representations $\hat a$ which are then decoded using a \textit{forward dynamics model} (FDM) $\hat\vx' = \hat \vg(\vx, \hat a)$. $\text{\textbf{Secondly}}$, the IDM is used to label the unlabeled video dataset, which yields $\{(\vx_i, \hat a_i, \vx_i')\}_{i=1}^N$ where ${\hat a_i := \arg\max_{\hat a} \hat q(\hat a \mid \vx_i, \vx'_i)}$. This dataset is then used to train the \textit{latent action policy} $\hat\pi(\hat a \mid \vx)$. \textbf{Thirdly}, a learnable head is applied on top of $\hat\pi(\hat a \mid \vx)$ and trained to map the latent actions $\hat a$ to actual actions $a$ using a much smaller domain-specific action-labeled dataset $\{(\vx_i, a_i)\}_{i = 1} ^{N_a}$ ($N_a << N$). While doing that, one can choose to either freeze the latent action policy $\hat\pi$ or fine-tune it. The authors showed that, thanks to this approach, fewer action-labeled samples are needed to train a good policy.\footnote{The authors also show $\hat \pi$ can be fine-tuned via reinforcement learning more efficiently.} Since then, this idea has been applied at larger scale in Genie~\citep{bruce2024genie} and augmented with textual goal-conditioning in LAPA~\citep{ye2025latent}.

%of the form $\normin{\vx' - \hat\vg(\vx, \hat\vf(\vx, \vx'))}^2$ where $\hat a := \hat\vf(\vx, \vx')$ is an \textit{inverse dynamics model} (IDM) and $\hat\vg(\vx, \hat a)$ is a \textit{forward dynamics model} (FDM). Second, the \textit{latent action policy} $\hat\pi(\hat a \mid \vx)$ is trained to predict the output of $\hat a = \hat \vf(\vx, \vx')$ on the unlabeled video data. 

Motivated by the growing importance of this framework, we propose to study the identifiability of LAPO. Although identifiability in representation learning has been the subject of recent research efforts~\citep{hyvarinen2023nonlinear,iVAEkhemakhem20a,vonkugelgen2021selfsupervised,lachapelle2022disentanglement,buchholz2022function,zhang2023identifiability}, to the best of our knowledge, identifiability in the context of latent action modeling has not been investigated. %This setting presents new technical challenges such as the discreteness of the learned representation $\hat a$ and the dependency on the current state $\vx$ in both the encoder $\hat q(\hat a \mid \vx, \vx')$ and the decoder $\hat\vg(\vx, \hat a)$.

\noindent\textbf{Contributions.} First, we postulate a data-generating process for the expert transitions $(\vx, a, \vx')$ (\Cref{sec:dgp}). Next, we provide formal desiderata for the IDM $\hat q(\hat a \mid \vx, \vx')$ to be useful and discuss statistical consequences (\Cref{sec:desiderata}). We further discuss two potential sources of unidentifiability (\Cref{sec:unident}) and, finally, we prove that an entropy-regularized LAPO objective (\Cref{sec:LAPO}) is guaranteed to identify an IDM satisfying the said desiderata, under suitable conditions (\Cref{sec:ident}).

\section{Identifiability analysis of LAPO}
\subsection{Data-generating process}\label{sec:dgp}
Let $\vx \in \gX := [0,1]^d$ be the current observation, $\vx' \in \gX' := [0,1]^{d'}$ be the future observation and $a \in \gA := \{1, \dots, k\}$ be a discrete action (in line with practical implementations~\citep{schmidt2024learning,bruce2024genie,ye2025latent}). The most natural situation is when $\vx$ and $\vx'$ corresponds to two consecutive frames, i.e. $\vx := \vx^t$ and $\vx' = \vx^{t+1}$ and $a = a^t$. But one could also consider different situations where the model is conditioned on a window of past observations: $\vx=\vx^{t-k:t}$ with $a=a^{t-k:t}$. Similarly, the $\vx'$ could correspond to a window of multiple frames in the future.

We assume the current state $\vx \in \gX$ is sampled from some (Lebesgue) density function $p(\vx)$. Furthermore, an action $a \in \gA$ is chosen according to a ground-truth policy $\pi$ conditioned on $\vx$:
\begin{align*}
    \vx \sim p(\vx) \quad\quad \text{and} \quad\quad a\sim \pi(a \mid \vx) \,.
\end{align*}
Define $p(\vx, a):= p(\vx)\pi(a \mid \vx)$, $p(a) := \int p(\vx, a)d\vx$, and $p(\vx \mid a) := p(\vx, a)/p(a)$.%, which are respectively the joint distribution of $(\vx, a)$, the marginal distribution of $a$ and the conditional density of $\vx$ given $a$.

We assume the future observation $\vx'$ is given by a deterministic transition model 
$$\vx' = \vg(\vx, a) \text{, where}\ \vg: \gX \times \gA \rightarrow \gX'\,.$$
This process induces a joint probability distribution over $(\vx, \vx')$, which we denote by $p(\vx, \vx')$.\footnote{$p(\vx, \vx')$ is an abuse of notation since the distribution of $(\vx, \vx')$ has no a density (w.r.t. Lebesgue).} %This is the distribution for the unlabeled (action free) video data. %We denote the distribution for the \textit{action-labeled} dataset by $\tilde p(\vx, a, \vx)$, which is induced by a potentially different distribution $\tilde p(\vx)$ and a potentially different policy ${\tilde \pi(a \mid \vx)}$, but the same deterministic transition model $\vg(\vx, a)$. 

\noindent\textbf{Support notation.} In what follows, the support of $p(\vx \mid a)$ is defined as
$$\textstyle \supp[p(\vx \mid a)] := \{\vx_0 \in \gX \mid \forall\ \text{open neighborhood $U$ of}\ \vx_0, \int_U p(\vx \mid a)d\vx > 0\}\,,$$
where $\int d\vx$ denotes the Lebesgue integral. Note that $\supp[p(\vx \mid a)]$ might depend on $a$. Define also $\supp[p(a)] := \{a \in \gA \mid p(a) > 0\}$ and $\supp[p(\vx, a)] := \bigcup_{a \in \supp[p(a)]}\supp[p(\vx \mid a)] \times \{a\}$.

\subsection{Desiderata}\label{sec:desiderata}
In this section, we formalize three desiderata for an action representation and discuss statistical efficiency. Intuitively, we want to learn an encoder $\hat q (\hat a \mid \vx, \vx')$ that captures useful information about the ground-truth action $a$. %As mentioned in the previous section, because of the zero entropy condition in \Cref{def:Q}, the IDM $\hat q(\hat a \mid \vx, \vx')$ is deterministic and thus concentrates its mass on some action $\hat a = \hat\vf(\vx, \vx')$. 
To formalize this, we will study
\begin{align*}
    \vv(\hat a \mid \vx, a) := \hat q (\hat a \mid \vx, \vg(\vx, a))\, ,
\end{align*}
defined for all $(\vx, a) \in \supp[p(\vx, a)]$. The conditional probability mass function $\vv(\hat a \mid \vx, a)$ maps pairs $(\vx, a)$ to their corresponding learned action representations $\hat a$, potentially in a stochastic way. This is effectively an \textit{entanglement map}, as studied in identifiable representation learning~\citep{lachapelle2024nonparametric}. Our first desideratum is to have a deterministic map from $a$ to $\hat a$.
\begin{desideratum}[Determinism]\label{desideratum1} There exists a function $\vv(\vx, a)$ such that $\vv(\hat a \mid \vx, a) = {\mathbf{1}(\hat a = \vv(\vx, a))}$ for all $(\vx, a) \in \supp[p(\vx, a)]$, where $\mathbf{1}(\cdot)$ is the indicator function. 
\end{desideratum}

Notice how the learned action representation $\hat a$ might depend on the current state $\vx$ via $\hat a = \vv(\vx, a)$. Such a dependence is undesirable since it signifies that the meaning of $\hat a$, i.e. how it relates to $a$, depends on the current state $\vx$. We illustrate this unfortunate state of affairs with a simple example.

\begin{example}
Consider a manipulation task where $\gA:=\{\texttt{left}, \texttt{right}\}$ and suppose 
\begin{align*}
    &1 = \vv(\vx_0, a = \texttt{left}), \quad 2 = \vv(\vx_0, a = \texttt{right}),\\
    &2 = \vv(\vx_1, a = \texttt{left}), \quad 1 = \vv(\vx_1, a = \texttt{right})\,.
\end{align*}
We can see that the meaning of $\hat a$ depends on the context $\vx$: When in state $\vx_0$, $\hat a = 1$ corresponds to $a = \texttt{left}$, whereas in state $\vx_1$, $\hat a = 1$ corresponds to $a = \texttt{right}$. 
\end{example}

This undesirable phenomenon, described informally by \citet[Section 6.2]{schmidt2024learning}, can be thought of as a form of entanglement since $\hat a$ entangles both $a$ and $\vx$. This motivates:
\begin{desideratum}[Disentanglement]\label{desideratum2}
    There exists a function $\vv(a)$ such that, for all $(\vx, a) \in \supp[p(\vx, a)]$, $\vv(\vx, a) = \vv(a)$.
\end{desideratum}
Furthermore, we want the latent action $\hat a$ to reveal all there is to know about the ground-truth action $a$. More formally, we want that  two distinct actions $a_1$ and $a_2$ never map to the same latent action $\hat a$:
\begin{desideratum}[Informativeness]\label{desideratum3}
    The function $\vv: \supp[p(a)] \rightarrow \hat\gA$ is injective.
\end{desideratum}

\noindent\textbf{Statistical efficiency.} As explained in \Cref{sec:intro}, an encoder/IDM $\hat q(\hat a \mid \vx, \vx')$  can be used to label the action-free video dataset, yielding $\{(\vx_i, \hat a_i, \vx'_i)\}^N_{i=1}$. If the IDM satisfies our desiderata, the newly labeled dataset is actually $\{(\vx_i, \vv(a_i), \vx'_i)\}^N_{i=1}$ where $a_i$ is the action taken by the expert policy $\pi(a \mid \vx_i)$. Thus the latent action policy $\hat \pi(\hat a \mid \vx)$ trained on this data will approximate %${\pi(\vv^{-1}(\hat a) \mid \vx)\mathbf{1}(\hat a \in \vv(\supp[p(a)]))}$, i.e. 
the distribution of $\vv(a)$ when $a \sim \pi(a \mid \vx)$. This means that there exists a transformation $\sigma: \hat\gA \rightarrow \gA$ (any extension of $\vv^{-1}: \vv(\supp[p(a)]) \rightarrow \gA$) such that $\sigma(\hat a) \sim \pi(a \mid \vx)$ when $\hat a  \sim \pi(\hat a \mid \vx)$. Hence, to get the expert policy $\pi$, we only need to learn $\sigma: \hat\gA \rightarrow \gA$ ``on top of'' the latent action policy $\hat\pi$ using the smaller action-labeled dataset. Had $\hat a$ been dependent on $\vx$, such a transformation $\sigma$ would not exist, forcing us to resort to either fine-tuning $\hat\pi$ or learning a map $\gX\times\hat\gA \rightarrow \gA$ on top of $\hat\pi$, both of which are expected to be less statistically efficient than learning the simpler function $\sigma: \hat\gA \rightarrow \gA$.

\subsection{A formal entropy-regularized LAPO objective}\label{sec:LAPO}
We now present a formal entropy-regularized LAPO objective. \Cref{thm:ident} will show that, under suitable assumptions, its solutions must satisfy the desiderata of \Cref{sec:desiderata}. 

In order to learn a deterministic encoder, we add an entropy regularizer $H(\hat q(\cdot \mid \vx, \vx')) := {-\sE_{\hat q(\hat a \mid \vx, \vx')}\log \hat q (\hat a \mid \vx, \vx')}$. In the limit of infinite data, our entropy-regularized LAPO objective is
\begin{align}
    \min_{\hat \vg \in \gG, \hat q \in \gQ} \sE_{p(\vx, \vx')} \left[\sE_{\hat q (\hat a \mid \vx, \vx')} \normin{\vx' - \hat\vg(\vx, \hat a)}^2_2 + \beta H(\hat q(\cdot \mid \vx, \vx')) \right] \, , \label{prob:population}
\end{align}
where $\beta > 0$ controls regularization, and $\gG$ and $\gQ$ are respectively the hypothesis spaces for $\hat \vg$ and $\hat q$.

\begin{definition}[FDM hypothesis space $\gG$] \label{def:G} A function $\hat\vg : \gX \times \hat \gA \rightarrow \gX'$ is in $\gG$ if and only if, for all $\hat a \in \hat\gA$, $\hat\vg(\vx, \hat a)$ is continuous in $\vx$.
\end{definition}

\begin{definition}[IDM hypothesis space $\gQ$] \label{def:Q} Let $\hat \gA := \{1, \dots, \hat k\}$ be the space of action representations $\hat a$. A function $\hat q: \hat\gA \times \gX \times \gX' \rightarrow [0,1]$ is in $\gQ$ if and only if (i) for all $(\vx, \vx') \in \gX\times\gX'$, ${\sum_{\hat a \in \hat\gA}\  \hat q (\hat a \mid \vx, \vx') = 1}$, and (ii) for all $\hat a \in \hat\gA$, $\hat q(\hat a \mid \vx, \vx')$ is continuous in $(\vx, \vx')$.
    %\item \textbf{[Zero entropy]} $\forall (\vx, \vx') \in \gX\times\gX',\ H(\hat q (\hat a \mid \vx, \vx')) = 0$. \seb{ISSUE!!! We should require it only on the support of the data!!! That's what the regularizer was doing... In fact, if the entropy is zero everywhere, $\hat a$ is always the same!!! So it doesn't really apply to temperature going to zero... Also fix \Cref{prop:min_existence} consequently...}
\end{definition}
Note that our identifiability guarantee, \Cref{thm:ident}, does not assume $\hat k = k$, only $\hat k \geq k$.

One can easily see that both terms in Problem \Cref{prob:population} are lower bounded by zero. Additionally, \Cref{prop:min_existence} (in appendix) shows that there exists $(\vg^*, q^*) \in \gG\times\gQ$ such that both terms are equal to zero (under \Cref{ass:continuous_g,ass:injective_g}). This means that, at optimality, the entropy regularizer must be equal to zero thus forcing the learned IDM $\hat q(\hat a \mid \vx, \vx')$ to be deterministic for all $(\vx, \vx') \in \supp[p(\vx, \vx')]$. In other words, at optimality, $\hat q(\hat a \mid \vx, \vx') = \mathbf{1}(\hat a = \hat \vf(\vx, \vx'))$ for some function $\hat \vf : \supp[p(\vx, \vx')] \rightarrow \hat\gA$. % the IDM $\hat q(\hat a \mid \vx, \vx')$ must be deterministic, i.e. $\hat q(\hat a \mid \vx, \vx') = \mathbf{1}(\hat a = \hat \vf(\vx, \vx'))$ for some function $\hat \vf : \supp[p(\vx, \vx')] \rightarrow \hat\gA$.

\begin{remark}
    The above development begs the question: Why are we considering a stochastic IDM $\hat q$ to later regularize it to be deterministic? A perhaps more natural route would be to directly train a deterministic encoder $\hat\vf: \gX \times \gX' \rightarrow \hat \gA$. From an optimization perspective, a stochastic encoder is helpful as it unlocks gradient computation via the reparameterization trick~\citep{jang2016categorical}. From a theoretical perspective, the continuity condition on $\hat q$ is crucial for our proof, as it excludes pathological encoders $\hat\vf$ that would present ``jumps'' on the connected components of $\supp[p(\vx, \vx')]$. We conjecture that the VQ-VAE approach of \citet{schmidt2024learning}, which is limited to deterministic discrete encoders, can in principle lead to such pathological behaviors. We leave this for future work. %Whether switching to this entropy-regularized stochastic encoder would improve performance in practice is %left open.%\seb{details in appendix?}
\end{remark}

\subsection{Potential sources of unidentifiability}\label{sec:unident}
We now show that, without assumptions on the data-generating process or without restrictions on the hypothesis classes $\gQ$ and $\gG$, Problem \Cref{prob:population} admits degenerate solutions which do not satisfy our desiderata of \Cref{sec:desiderata}.

\begin{example}[No restriction on $\hat\gA$] In principle, one can choose $\hat\gA := \gX'$ and  $\hat q(\hat a \mid \vx, \vx') := \delta(\hat a - \vx')$ where $\delta$ is the Dirac function. Hence the IDM outputs $\vx'$ deterministically. By choosing $\hat\vg(\vx, \hat a) = \hat a$, we clearly solve Problem~\cref{prob:population}, but the action representation $\hat a$ is uninteresting. In fact, this can be understood as a violation of \Cref{desideratum2} since $\hat a = \vx'$ clearly depends on $\vx$ via $\vx' = \vg(\vx, a)$.
\end{example}

\begin{example}[Deterministic $\pi(a \mid \vx)$] Assume $\pi(a \mid \vx) = \mathbf{1}(a = \pi(\vx))$, i.e. the ground-truth policy is deterministic. In that case, one can solve the reconstruction problem by choosing $\hat\vg(\vx, \hat a) := \vg(\vx, \pi(\vx))$ since $\vx' = \vg(\vx, \pi(\vx))$ with probability one. In that case, the latent action $\hat a$ is completely ignored by the FDM and  thus the IDM $\hat q(\hat a \mid \vx, \vx')$ could simply output the same action deterministically, which would clearly present a violation of \Cref{desideratum3}.
\end{example}

\subsection{Main identifiability result}\label{sec:ident}
\begin{figure}
    \centering
    \includegraphics[width=1\linewidth]{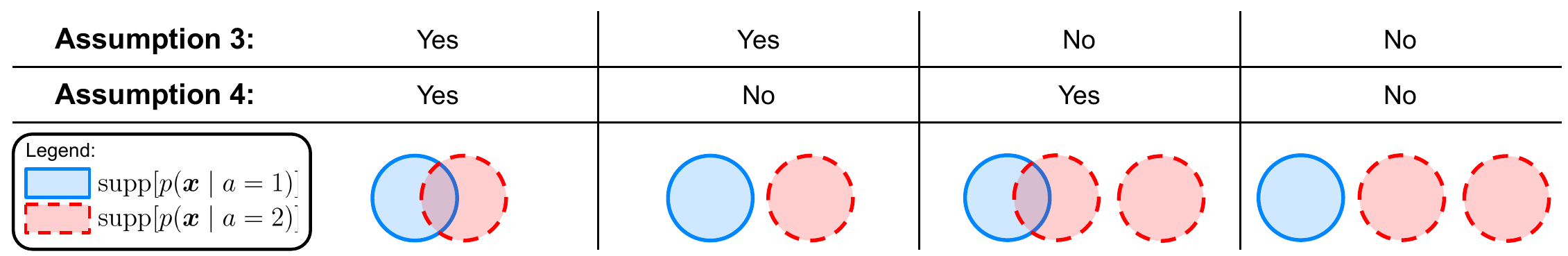}
    \caption{Illustration of \Cref{ass:connected_p(x|a),ass:intersect_p(x|a)}. Assume $\gA := \{1, 2\}$}
    \label{fig:topo_assump}
\end{figure}
In this section, we provide sufficient conditions on the data-generating process under which the solutions of Problem \Cref{prob:population} are guaranteed to satisfy the desiderata of \Cref{sec:desiderata}.

%We assume that $\pi(a \mid \cdot)$ %and $\tilde\pi(a \mid \cdot)$ are
%is Lebesgue measurable for all $a \in \gA$. For $\vg$, we require more assumptions:
First of all, we require the ground-truth FDM to be continuous.
\begin{assumption}[Continuous $\vg$]\label{ass:continuous_g}
    For all $a \in \gA$, the ground-truth FDM $\vg(\vx, a)$ is continuous in $\vx$.
\end{assumption}
Additionally, we require that different actions always have different effects in the data-generating process. We formalize this as a form of injectivity.
\begin{assumption}[Injectivity] \label{ass:injective_g}
    For all $\vx \in \gX$ and $a_1, a_2 \in \gA$, if $a_1 \not= a_2$, then $\vg(\vx, a_1) \not= \vg(\vx, a_2)$.
\end{assumption} 

The last two assumptions put topological restrictions on the support of $p(\vx, a)$. Recall that a set $S \subseteq \sR^d$ is said to be \textit{connected} if it ``holds in one piece''~\citep{Munkres2000Topology}. See \Cref{fig:topo_assump} for an illustration.%are illustrated in \Cref{fig:technical_ass}. \seb{TODO}
\begin{assumption}\label{ass:connected_p(x|a)}
For all ${a \in \supp[p(a)]}$, we have that $\supp[p(\vx \mid a)]$ is a connected subset of $\gX$. 
\end{assumption}

\begin{assumption} \label{ass:intersect_p(x|a)}
For all pairs $a_1, a_2 \in \supp[p(a)]$, $\supp[p(\vx \mid a_1)] \cap \supp[p(\vx \mid a_2)] \not= \emptyset$.
\end{assumption}

Note that \Cref{ass:connected_p(x|a),ass:intersect_p(x|a)} are both satisfied for example if $\supp[p(\vx, a)] = \gX \times \gA$.

We are now ready to state the main identifiability result of this work. It shows that, under the assumptions introduced above, the encoder/IDM $\hat q(\hat a \mid \vx, \vx')$ learned by optimizing Problem~\Cref{prob:population} must satisfy the desiderata of \Cref{sec:desiderata}. Its proof can be found in the appendix. Recall ${\vv(\hat a \mid \vx, a) := \hat q(\hat a \mid \vx, \vg(\vx, a))}$.

\begin{restatable}{theorem}{mainThm}\label{thm:ident}
    Suppose $\hat k \geq k$ and let $(\hat \vg, \hat q)$ be a solution\footnote{Under \Cref{ass:continuous_g,ass:injective_g}, a solution is guaranteed to exist by \Cref{prop:min_existence} in appendix.} of Problem~\Cref{prob:population} with hypothesis classes $\gG$ (\Cref{def:G}) and $\gQ$ (\Cref{def:Q}). 
    \begin{enumerate}
        \item If \Cref{ass:continuous_g,ass:injective_g} hold, then \Cref{desideratum1} holds, i.e. there exists a function ${\vv: \supp[p(\vx, a)] \rightarrow \hat\gA}$ such that, for all $(\vx, a) \in \supp[p(\vx, a)]$, 
        $$\vv(\hat a \mid \vx, a) = \mathbf{1}(\hat a = \vv(\vx, a))\,.$$
        \item If \Cref{ass:continuous_g,ass:injective_g,ass:connected_p(x|a)} hold, then \Cref{desideratum2} holds, i.e. there exists a mapping ${\vv: \supp[p(a)] \rightarrow \hat \gA}$ such that, for all $(\vx, a) \in \supp[p(\vx, a)]$, 
        $$\vv(\hat a \mid \vx, a) = \mathbf{1}(\hat a = \vv(a))\,.$$
        \item %\seb{Up to now, the reconstruction loss has not been used!}
        If \Cref{ass:continuous_g,ass:injective_g,ass:connected_p(x|a),ass:intersect_p(x|a)} hold, then \Cref{desideratum3} holds, i.e. the mapping ${\vv: \supp[p(a)] \rightarrow \hat\gA}$ defined above is injective.
    \end{enumerate}
\end{restatable}

\printbibliography
%\bibliographystyle{plain}
%\bibliography{ref}

%%%%%%%%%%%%%%%%%%%%%%%%%%%%%%%%%%%%%%%%%%%%%%%%%%%%%%%%%%%%
\appendix
\newpage
\section*{Appendix}
Multiple intermediary results are necessary in order to prove \cref{thm:ident}. The proof of the following technical lemma can be safely skipped at first read.

\begin{lemma}\label{lem:q_construct}
    Let $X$ be a metric space\footnote{The result generalizes to the case where $X$ is a perfectly normal topological space.} and let $X_1, X_2, ..., X_K$ be a finite collection of disjoint closed sets of $X$. Then, there exists a function $q: [K] \times X \rightarrow [0,1]$ such that
    \begin{itemize}
        \item for all $x \in X$, $\sum_{k\in [K]} q(k \mid x) = 1$,
        \item for all $k \in [K]$, $q(k \mid x)$ is continuous in $x$ (as a function $X \rightarrow [0,1]$),  and
        \item for all $k \in [K]$ and all $x \in X_k$, $q(k \mid x) = 1$.
    \end{itemize}
\end{lemma}
\begin{proof}
    We make use of the Vedenissoff theorem~\citep[Theorem 1.5.19]{Eng89}. We extract only the part of the theorem we will need: If a topological space $X$ is perfectly normal, then for every pair of disjoint closed sets $A, B \subseteq X$, there exists a continuous function $f:X \rightarrow [0,1]$ such that $f^{-1}(\{0\}) = A$ and $f^{-1}(\{1\}) = B$.

    Since a metric space is always perfectly normal~\citep[Corollary 4.1.13]{Eng89}, $X$ is perfectly normal and thus we can apply the Vedenissoff theorem.
    
    For each $k \in [K]$, the set $\bigcup_{k' \in [K]\setminus \{k\}} X_{k'}$ is closed since a finite union of closed sets is closed. By Vedenissoff theorem, there exists a continuous function $h_k: X \rightarrow [0,1]$ such that $h_k^{-1}(\{1\}) = X_k$ and $h_k^{-1}(\{0\}) = \bigcup_{k' \in [K]\setminus \{k\}} X_{k'}$.

    We now prove that $\sum_{k \in [K]} h_k(x) > 0$ for all $x \in X$. We consider two cases, $x \in \bigcup_{k\in [K]} X_k$ and $x \not\in \bigcup_{k\in [K]} X_k$. If $x \in \bigcup_{k\in [K]} X_k$, then there exists a $k_0 \in [K]$ such that $x \in X_{k_0} = h_{k_0}^{-1}(\{1\})$ which means $h_{k_0}(x) = 1$. Of course, this implies that the sum is greater than zero. If $x \not\in \bigcup_{k\in [K]} X_k$, in particular we have $x \not\in \bigcup_{k\in [K]\setminus \{1\}} X_k = h_1^{-1}(\{0\})$, which means $h_1(x) > 0$. Of course, this implies the sum is greater than zero.

    Since $\sum_{k \in [K]} h_k(x) > 0$ for all $x \in X$, we can define, for all $(k, x) \in [K]\times X$,  
    $$q(k \mid x) := \frac{h_k(x)}{\sum_{{k'} \in [K]}h_{k'}(x)}\, .$$

    We now verify that $q$ satisfies the three conditions of the theorem. First, it is clear that $\sum_{k \in [K]} q(k \mid x) = 1$. Second, $q(k \mid x)$ is continuous in $x$ since all $h_{k'}$ are continuous functions. 
    
    Third, we check that $q(k \mid x) = 1$ when $x \in X_k$. Let $k \in [K]$ and $x \in X_k$. Since $X_k = h^{-1}_k(\{1\})$, we have that $h_k(x) = 1$. Consider $k' \in [K] \setminus \{k\}$. Clearly, $x \in X_k \subset \bigcup_{k'' \in [K]\setminus \{k'\}} X_{k''} = h^{-1}_{k'}(\{0\})$, which means $h_{k'}(x) = 0$ for $k' \not= k$. We thus have 
    $$q(k \mid x) =  \frac{h_k(x)}{h_k(x) + \sum_{{k'} \in [K]\setminus \{k\}} h_{k'}(x)} = \frac{1}{1 + 0} = 1 \,,$$
    which concludes the proof.
\end{proof}

Define the function $G: \gX \times \gA \rightarrow \gX \times \gX'$ as $G(\vx, a) := (\vx, \vg(\vx, a))$, which is simply the function $\vg(\vx, a)$ with a copy of $\vx$ in its output. Note that its image $G(\gX\times \gA)$ is effectively the set of plausible transition pairs $(\vx, \vx')$. In general, this is expected to be a proper subset of $\gX \times \gX'$. 

\begin{lemma}
    Under \Cref{ass:injective_g}, $G$ is injective.
\end{lemma}
\begin{proof}
    If $G(\vx_1, \va_1) = G(\vx_2, \va_2)$, then $\vx_1 = \vx_2$ and $\vg(\vx_1, a_1) = \vg(\vx_1, a_2)$ and thus, by \cref{ass:injective_g}, $a_1 = a_2$.
\end{proof}

\begin{remark}\label{rem:f}
    Since $G$ is injective, it is bijective on its image $G(\gX \times \gA) \subseteq \gX \times \gX'$. Therefore, it has an inverse $F:G(\gX \times \gA) \rightarrow \gX \times \gA$ which clearly has the form $F(\vx, \vx') = (\vx, \vf(\vx, \vx'))$, for some function $\vf: G(\gX \times \gA) \rightarrow \gA$. 
\end{remark}

%It will be handy to define
%\begin{align}
%    \gP := \supp[p(\vx, a)] \subset \gX \times \gA \,,
%\end{align}
%which is essentially the set of possible pairs $(\vx, a)$ under the DGP defined in \cref{sec:dgp}.

\begin{proposition}\label{prop:min_existence}
    Suppose \cref{ass:continuous_g,ass:injective_g} hold and $\hat k \geq k$. Then, there exist $\vg^* \in \gG$ and $q^* \in \gQ$ such that the objective of Problem~\cref{prob:population} is equal to zero. 
\end{proposition}
\begin{proof}
    Take $\vg^*(\vx, \hat a) := \mathbf{1}(\hat a \in \gA)\vg(\vx, \hat a)$. Essentially, $\vg^*$ imitates the ground-truth FDM $\vg$ when the action $\hat a \in \gA$, otherwise it simply outputs zero. Clearly, $\vg^* \in \gG$ since, by \cref{ass:continuous_g}, $\vg(\cdot, \hat a)$ is continuous for all $\hat a \in \gA$ (and the zero function is continuous).  

    We now construct a $q^* \in \gQ$ such that, for all $(\vx, \vx') \in G(\gX \times \gA)$ and all $\hat a \in \hat \gA$, we have $q^*(\hat a \mid \vx, \vx') = \delta(\vf(\vx, \vx') - \hat a)$, where $\vf(\vx, \vx')$ is defined in \cref{rem:f}. Later on, we show that the pair $(q^*, \vg^*)$ sets the objective to zero.

    Notice that $G(\gX \times \gA) = \bigcup_{a \in \gA} G(\gX\times \{a\})$ where $G(\gX\times \{a\}) = \{(\vx, \vg(\vx, a)) \mid \vx \in \gX \}$ is the graph of $\vg(\cdot, a)$. Since $\vg(\cdot, a)$ is continuous by \cref{ass:continuous_g}, the closed graph theorem implies that its graph, $G(\gX\times \{a\})$, is closed is $\gX \times \gX'$. Furthermore, we know that the sets $G(\gX\times \{a\})$ are mutually disjoint since otherwise there exists $(\vx, \vx') \in G(\gX\times \{a_1\}) \cap G(\gX\times \{a_2\})$ for distinct $a_1$, $a_2$ which implies $(\vx, \vg(\vx, a_1)) = (\vx, \vx') = (\vx, \vg(\vx, a_2))$, which violates \cref{ass:injective_g}. 
    
    To summarize, the last paragraph showed that $\{G(\gX \times \{a\})\}_{a \in \gA}$ is a partition of $G(\gX \times \gA)$ where each $G(\gX \times \{a\})$ is closed in $\gX \times \gX'$. By noticing that $\gX\times \gX'$ is a metric space, we can apply \cref{lem:q_construct} to show the existence of a function $q: \gA \times \gX \times \gX' \rightarrow [0,1]$ such that the following holds:
    \begin{itemize}
        \item for all $(\vx, \vx') \in \gX \times \gX'$, $\sum_{a \in \gA} q(a \mid \vx, \vx') = 1$,
        \item for all $a \in \gA$, $q(a \mid \vx, \vx')$ is continuous in $(\vx, \vx')$, and 
        \item for all $a \in \gA$ and all $(\vx, \vx') \in G( \gX \times \{a\})$, $ q(a \mid \vx, \vx') = 1$.
    \end{itemize}
    We choose, for all $(\hat a, \vx, \vx') \in \hat\gA \times \gX \times \gX'$, $q^*(\hat a \mid \vx, \vx') := \mathbf{1}(\hat a \in \gA) q(\hat a \mid \vx, \vx')$. In other words, $q^*$ imitates $q$ when $\hat a \in \gA$, and simply outputs zero when $\hat a \not\in \gA$. 
    
    We now check that $q^* \in \gQ$. Take $(\vx, \vx') \in \gX\times \gX'$. We have
    \begin{align}
        \sum_{\hat a \in \hat \gA} q^*(\hat a \mid \vx, \vx') = \sum_{\hat a \in \hat \gA} \mathbf{1}(\hat a \in \gA)q(\hat a \mid \vx, \vx') = \sum_{\hat a \in \gA} q(\hat a \mid \vx, \vx') = 1 \,.\,
    \end{align}
    where the second equality used the fact that $\gA \subseteq \hat\gA$ (since $\hat k \geq k$).
    
    Now take $\hat a \in \hat \gA$. If $\hat a \in \gA$, then $q^*(\hat a \mid \vx, \vx') = q(a \mid \vx, \vx')$ which is continuous in $(\vx, \vx')$. If $\hat a \not\in \gA$, then $q^*(\hat a \mid \vx, \vx') = 0$ which is also continuous. Thus $q^* \in \gQ$.

    Notice that, for all $(\vx, a) \in \gX\times\gA$, we have
    \begin{align}
        q^*(\hat a \mid \vx, \vg(\vx, a)) = \mathbf{1}(\hat a \in \gA)q(\hat a \mid \vx, \vg(\vx, a)) = \mathbf{1}(\hat a \in \gA)\mathbf{1}(\hat a = a) = \mathbf{1}(\hat a = a) \,,
    \end{align}
    where the third equality holds because $(\vx, \vg(\vx, a)) \in G(\gX \times \{a\})$, which implies that ${q(a \mid \vx, \vg(\vx, a)) = 1}$.

    Now, we must show that the pair $(\vg^*, q^*)$ sets the loss of Problem~\cref{prob:population} to zero. First note that 
    \begin{align}
        \sE_{p(\vx, \vx')} H(q^*(\cdot \mid \vx, \vx')) &= \sE_{p(\vx, a)} H(q^*(\cdot \mid \vx, \vg(\vx, a))) \\
        &= \sE_{p(\vx, a)} H(\mathbf{1}( \cdot = a)) \\
        &= \sE_{p(\vx, a)} 0 = 0 \,.
    \end{align}

    Also, 
    \begin{align}
        \sE_{p(\vx, \vx')} \sum_{\hat a \in \hat\gA} &q^*(\hat a \mid \vx, \vx') \normin{\vx' - \vg^*(\vx, \hat a)}^2_2 \\
        &= \sE_{p(\vx, \vx')} \sum_{\hat a \in \hat\gA} q^*(\hat a \mid \vx, \vx') \normin{\vx' - \mathbf{1}(\hat a \in \gA)\vg(\vx, \hat a)}^2_2 \\
        &= \sE_{p(\vx, a)} \sum_{\hat a \in \hat\gA} q^*(\hat a \mid \vx, \vg(\vx, a)) \normin{\vg(\vx, a) - \mathbf{1}(\hat a \in \gA)\vg(\vx, \hat a)}^2_2 \\
        &= \sE_{p(\vx, a)} \sum_{\hat a \in \hat\gA} \mathbf{1}(\hat a = a) \normin{\vg(\vx, a) - \mathbf{1}(\hat a \in \gA)\vg(\vx, \hat a)}^2_2 \\
        &= \sE_{p(\vx, a)} \normin{\vg(\vx, a) - \mathbf{1}(a \in \gA)\vg(\vx, a)}^2_2 \\ 
        &= \sE_{p(\vx, a)} \normin{\vg(\vx, a) - \vg(\vx, a)}^2_2 \\
        &= \sE_{p(\vx, a)} 0 = 0 \,,
    \end{align}    
    where the fourth equality used the fact that $\gA \subseteq \hat\gA$, which holds since $\hat k \geq k$. This concludes the proof.
\end{proof}

The following lemma simply states that if the integral of a non-negative continuous function $f$ w.r.t. to some measure $\mu$ is equal to zero, then the function must be zero on the support of $\mu$. 
\begin{lemma}\label{lem:ae2everywhere}
    Let $(X, \tau)$ be a topological space and let $\gF$ be the Borel sigma-algebra for $X$. Let $\mu: \gF \rightarrow [0, \infty)$ be a measure and let $f:X\rightarrow [0,\infty )$ be a non-negative continuous function. If $\int f d\mu = 0$, then $f(x) = 0$ for all $x \in \supp[\mu]$.
\end{lemma}
\begin{proof}
    We show the contrapositive statement. Suppose there exists $x_0 \in \supp[\mu]$ such that $f(x_0) > 0$. Since $f$ is continuous, we have that $f^{-1}((0, \infty))$ is an open neighborhood of $x_0$. Since $x_0 \in \supp[\mu]$, we have that $\mu(f^{-1}((0, \infty))) >0$. But this means $\int f d\mu > 0$ \citep[Section 3A, Exercise 3]{axler2019measure}.
\end{proof}

We are finally ready to prove \cref{thm:ident}.

\mainThm*

\begin{proof}
    If $(\hat \vg, \hat q)$ solves Problem~\cref{prob:population}, we must have $\hat \vg \in \gG$ and $\hat q \in \gQ$. Moreover, since \cref{ass:continuous_g,ass:injective_g} hold and $\hat k \geq k$, we can apply \cref{prop:min_existence} to conclude that there exists a pair $(\vg^*, q^*) \in \gG \times \gQ$ that reaches zero loss. From this, we conclude that $(\hat \vg, \hat q)$ must also reach zero loss, otherwise it is not optimal.

    Thus we have 
    \begin{align}
        \sE_{p(\vx, \vx')} \left[\sum_{\hat a \in \hat \gA} \hat q(\hat a \mid \vx, \vx') \normin{\vx' - \hat\vg(\vx, \hat a)}^2_2 + H(\hat q(\cdot \mid \vx, \vx')) \right] = 0 \,.
    \end{align}
    Since both terms are lower bounded by 0, both terms must equal zero. We start by using the fact that the entropy term is equal to zero:
    \begin{align}
        \sE_{p(\vx, \vx')} H(\hat q(\cdot \mid \vx, \vx')) &= 0 \\
        \sE_{p(\vx, a)} H(\hat q(\cdot \mid \vx, \vg(\vx, a))) &= 0 \\
        %&\sE_{p(\vx, a)} -\sum_{\hat a \in \hat \gA}\hat q(\hat a \mid \vx, \vg(\vx, a)) \log \hat q(\hat a \mid \vx, \vg(\vx, a)))  = 0 \\
        \sE_{p(a)} \sE_{p(\vx \mid a)} H(\hat q(\cdot \mid \vx, \vg(\vx, a)))  &= 0 \\
        \sum_{a \in \supp[p(a)]} p(a)\sE_{p(\vx \mid a)} H(\hat q(\cdot \mid \vx, \vg(\vx, a)))  &= 0 \,,
    \end{align}
    where $p(\vx, a) := p(\vx)\pi(a \mid \vx)$, $p(a) := \int p(\vx, a) d\vx$ and $p(\vx \mid a) := p(\vx, a) / p(a)$. The l.h.s. is a sum of positive terms. We thus have, for each $a \in \supp[p(a)]$,
    \begin{align}
    \sE_{p(\vx \mid a)} H(\hat q(\cdot \mid \vx, \vg(\vx, a))) &= 0\,.
    \end{align}
    Since $H(\hat q(\cdot \mid \vx, \vg(\vx, a)))$ is greater or equal to zero and is a continuous function of $\vx$ (it follows from the continuity of $\vg(\vx, a)$, $\hat q(\hat a \mid \vx, \vx')$ and $ y \mapsto y \log y$\footnote{In the definition of entropy, $0 \log 0$ is defined to be equal to zero, which makes $y \mapsto y\log y$ a continuous function on $[0, \infty)$ since $\lim_{y \to 0^+} y \log y = 0$.}), \cref{lem:ae2everywhere} implies that $H(\hat q(\cdot \mid \vx, \vg(\vx, a))) = 0$, for all $\vx \in \supp[p(\vx \mid a)]$.
    %\begin{align}
    %    H(\hat q(\cdot \mid \vx, \vg(\vx, a))) &= 0 \\
    %    -\sum_{\hat a \in \hat \gA}\hat q(\hat a \mid \vx, \vg(\vx, a))) \log \hat q(\hat a \mid \vx, \vg(\vx, a)))  &= 0 \,. \label{eq:kdn339d}
    %\end{align}
    %Since the map $\vg(\vx, a)$ is continuous in $\vx$, since $ y \mapsto y \log y$ is continuous\footnote{In the definition of entropy, $0 \log 0$ is defined to be equal to zero, which makes $y \mapsto y\log y$ a continuous function on $[0, \infty)$ since $\lim_{y \to 0^+} y \log y = 0$.} on $[0,\infty)$ and since $\hat q(\hat a \mid \vx, \vx')$ is continuous in $(\vx, \vx')$ for all $\hat a \in \hat\gA$, we have that the l.h.s. of \cref{eq:kdn339d} is a continuous function of $\vx$. We can thus apply \cref{lem:ae2everywhere} to conclude that the equality holds \textit{for all} $\vx \in \supp[p(\vx \mid a)]$, instead of just ``almost all''.
    
    To summarize, we showed that, for all $a \in \supp[p(a)]$ and all $\vx \in \supp[p(\vx \mid a)]$, we have that ${H(\hat q(\cdot \mid \vx, \vg(\vx, a))) = 0}$. Since $$\bigcup_{a \in \supp[p(a)]} \supp[p(\vx \mid a)] \times \{a\} = \supp[p(\vx, a)]\,,$$
    it is equivalent to saying that, for all $(\vx, a) \in \supp[p(\vx, a)]$, we have ${H(\hat q(\cdot \mid \vx, \vg(\vx, a))) = 0}$. This means there exists a function $\vv: \supp[p(\vx, a)] \rightarrow \hat \gA$ such that, for all $(\vx, a) \in \supp[p(\vx, a)]$ and all $\hat a \in \hat \gA$, 
    \begin{align}
    \hat q(\hat a \mid \vx, \vg(\vx, a)) = \mathbf{1}(\hat a = \vv(\vx, a))\, , \label{eq:dirac_q} 
    \end{align}
    which proves the first statement.
    
    To prove the second statement, we rewrite the above equation as
    \begin{align}
    \hat q(\hat a \mid G(\vx, a)) = \mathbf{1}(\hat a = \vv(\vx, a))\, , \label{eq:dirac_q2} 
    \end{align}
    where $G:\gX\times\gA \rightarrow \gX\times\gX'$ was previously defined as $G(\vx, a):=(\vx, \vg(\vx, a))$. For each pair $(a, \hat a) \in \gA \times \hat\gA$, define the function $\hat q_{a, \hat a}: \gX \rightarrow [0,1]$ as $\hat q_{a, \hat a}(\vx) := \hat q(\hat a \mid G(\vx, a))$. Since $\hat q_{a, \hat a}$ is the composition of two continuous functions, namely $G(\cdot, a)$ and $\hat q(\hat a \mid \cdot, \cdot)$, it must also be continuous. We rewrite \cref{eq:dirac_q2} as follows: for all $\hat a \in \hat\gA$, all $a \in \supp[p(a)]$ and all $\vx \in \supp[p(\vx \mid a)]$, we have
    \begin{align}
    \hat q_{a, \hat a}(\vx) = \mathbf{1}(\hat a = \vv(\vx, a))\, , \label{eq:dirac_q3} 
    \end{align}
    Now, fix $\hat a\in \hat\gA$ and $a \in \supp[p(a)]$. It is clear from the above equation that $\hat q_{a, \hat a}(\supp[p(\vx \mid a)]) \subseteq \{0,1\}$. But since $\supp[p(\vx \mid a)]$ is connected (\Cref{ass:connected_p(x|a)}) and $\hat q_{a, \hat a}$ is continuous, we know that the image $\hat q_{a, \hat a}(\supp[p(\vx \mid a)])$ must also be connected. This implies that $\hat q_{a, \hat a}(\supp[p(\vx \mid a)])$ is either $\{0\}$ or $\{1\}$. %This means that, on $\supp[p(\vx \mid a)]$, the function $\hat q_{a, \hat a}$ is constant. 
    Thus, there is a function $\phi: \supp[p(a)] \times \hat\gA \rightarrow \{0,1\}$ that outputs the value that $\hat q_{a, \hat a}$ uniformly takes on $\supp[p(\vx \mid a)]$. In other words, for all $\hat a \in \hat\gA$, all $a \in \supp[p(a)]$ and all $\vx \in \supp[p(\vx \mid a)]$, we have
    \begin{align}
        \hat q_{a, \hat a}(\vx) = \phi(a, \hat a) \\
        \mathbf{1}(\hat a = \vv(\vx, a)) = \phi(a, \hat a) \,.
    \end{align}
    The last equation above implies that $\vv(\vx, a)$ is constant in $\vx$, for all values of $a \in \supp[p(a)]$. This means there is a function $\vv: \supp[p(a)] \rightarrow \hat\gA$ such that, for all $a \in \supp[p(a)]$ and all $\vx \in \supp[p(\vx \mid a)]$,
    \begin{align}
        \vv(\vx, a) = \vv(a) \,, \label{eq:sigma}
    \end{align}
    which shows the second statement.
    
    To prove the third statement ($\vv(a)$ injective), we leverage the fact that the reconstruction term is equal to zero:
    \begin{align}
        &\sE_{p(\vx, \vx')} \sum_{\hat a \in \hat \gA} \hat q(\hat a \mid \vx, \vx') \normin{\vx' - \hat\vg(\vx, \hat a)}^2_2 = 0 \\
        &\sE_{p(\vx, a)} \sum_{\hat a \in \hat \gA} \hat q(\hat a \mid \vx, \vg(\vx, a)) \normin{\vg(\vx, a) - \hat\vg(\vx, \hat a)}^2_2 = 0 \\
        &\sE_{p(a)} \sE_{p(\vx \mid a)} \sum_{\hat a \in \hat \gA} \hat q(\hat a \mid \vx, \vg(\vx, a)) \normin{\vg(\vx, a) - \hat\vg(\vx, \hat a)}^2_2 = 0 \\
    \end{align}
    The term inside the expectation $\sE_{p(a)}$ are greater or equal to zero, thus each of them must be equal to zero, i.e. for all $a \in \supp[p(a)]$, we have
    \begin{align}
        \sE_{p(\vx \mid a)} \sum_{\hat a \in \hat \gA} \hat q(\hat a \mid \vx, \vg(\vx, a)) \normin{\vg(\vx, a) - \hat\vg(\vx, \hat a)}^2_2 = 0 \,.
    \end{align}
    Since the inside of the expectation is always greater or equal to zero and is a continuous function of $\vx$, \cref{lem:ae2everywhere} implies that, for all $\vx \in \supp[p(\vx \mid a)]$. % we must have that, for \textit{almost all} $\vx \in \supp[p(\vx \mid a)]$,
    \begin{align}
        \sum_{\hat a \in \hat \gA} \hat q(\hat a \mid \vx, \vg(\vx, a)) \normin{\vg(\vx, a) - \hat\vg(\vx, \hat a)}^2_2 = 0 \,. \label{eq:pdmss093}
    \end{align}
    %Since all functions appearing above are continuous in $\vx$, we have that the l.h.s. is continuous in $\vx$ and thus, by \cref{lem:ae2everywhere}, we have equality \textit{for all} $\vx \in \supp[p(\vx \mid a)]$, instead of ``for almost all''. 
    To summarize, we showed that, for all $a \in \supp[p(a)]$ and all $\vx \in \supp[p(\vx \mid a)]$, \cref{eq:pdmss093} holds. We can thus derive that 
\begin{align}
        0 = \sum_{\hat a \in \hat \gA} \hat q(\hat a \mid \vx, \vg(\vx, a)) \normin{\vg(\vx, a) - \hat\vg(\vx, \hat a)}^2_2 &= \sum_{\hat a \in \hat \gA} \mathbf{1}(\hat a = \vv(\vx,a)) \normin{\vg(\vx, a) - \hat\vg(\vx, \hat a)}^2_2  \\
        &= \sum_{\hat a \in \hat \gA} \mathbf{1}(\hat a = \vv(a)) \normin{\vg(\vx, a) - \hat\vg(\vx, \hat a)}^2_2 \\
        &= \normin{\vg(\vx, a) - \hat\vg(\vx, \vv(a))}^2_2\,
\end{align}
where the first line leverages \cref{eq:dirac_q} and the second line uses \cref{eq:sigma}.
This means that, for all $a \in \supp[p(a)]$ and all $\vx \in \supp[p(\vx \mid a)]$,
\begin{align}
    \vg(\vx, a) = \hat \vg(\vx, \vv(a)) \,.
\end{align}
We now show that $\vv: \supp[p(a)] \rightarrow \hat\gA$ is injective. We proceed by contradiction. Suppose it is not injective. This means there exist two distinct $a_1, a_2 \in \supp[p(a)]$ such that $\vv(a_1) = \vv(a_2)$. By \cref{ass:intersect_p(x|a)}, we know there exists an $\vx_0$ that is in both $\supp[p(\vx \mid a_1)]$ and $\supp[p(\vx \mid a_2)]$. We note that
\begin{align}
    \vg(\vx_0, a_1) = \hat\vg(\vx_0, \vv(a_1)) = \hat\vg(\vx_0, \vv(a_2)) =  \vg(\vx_0, a_2) \,,
\end{align}
where the first equality holds because $\vx_0 \in \supp[p(\vx \mid a_1)]$ and the last equality holds because $\vx_0 \in \supp[p(\vx \mid a_2)]$. But this is contradicting \cref{ass:injective_g}. Thus, $\vv(a)$ is injective. 
\end{proof}

\end{document}